\documentclass[9pt,conference]{IEEEtran}
\IEEEoverridecommandlockouts
\pdfoutput=1

\usepackage{color}
\usepackage{bm}
\usepackage{hyperref}
\usepackage{color}
\usepackage{amsthm}
\usepackage{times,amsmath}
\usepackage{amssymb}
\usepackage{subcaption}
\usepackage{caption}
\usepackage{cite}
\usepackage{url}
\usepackage[capitalise]{cleveref}
\usepackage{enumitem}
\usepackage{tikz, pgfplots}
\usepackage{algorithm,algpseudocode}
\usepackage{moresize} 
\input{mysymbol.sty}

\def \Hgcn {{\bbH_{\mathrm{GCN}}}}

\usetikzlibrary{shapes,arrows}
\pgfplotsset{compat=1.17}
\pgfplotstableset{col sep=comma}
\usepgfplotslibrary{groupplots}

\tikzset{every mark/.append style={scale=1.6, solid}, font=\small}
\pgfplotsset{
    width=1\textwidth,
    height=5.5cm,
    legend style={
        font=\small ,  
        inner xsep=1pt,
        inner ysep=1pt,
        nodes={inner sep=1pt}},
    legend cell align=left,
    every axis/.append style={line width=.5pt},
 	every axis plot/.append style={line width=1.5pt},
 	every axis y label/.append style={yshift=-4pt}
}

\begin{document}

\title{Redesigning graph filter-based GNNs to relax the homophily assumption
\thanks{This work was supported in part by the Spanish AEI Grants PID2022-136887NB-I00, PID2023-149457OB-I00, FPU20/05554, the Community of Madrid (Madrid ELLIS Unit), and the NSF under award CCF-2340481.
Research was sponsored by the Army Research Office and was accomplished under Grant Number W911NF-17-S-0002. The views and conclusions contained in this document are those of the authors and should not be interpreted as representing the official policies, either expressed or implied, of the Army Research Office or the U.S. Army or the U.S. Government. The U.S. Government is authorized to reproduce and distribute reprints for Government purposes notwithstanding any copyright notation herein.
        Emails:  
        \href{mailto:samuel.rey.escudero@urjc.es}{samuel.rey.escudero@urjc.es}, 
        \href{mailto:nav@rice.edu}{nav@rice.edu},
        \href{mailto:victor.tenorio@urjc.es}{victor.tenorio@urjc.es}, 
        \href{mailto:segarra@rice.edu}{segarra@rice.edu},
        \href{mailto:antonio.garcia.marques@urjc.es}{antonio.garcia.marques@urjc.es}  
        }
}

\author{
\IEEEauthorblockN{
Samuel Rey\IEEEauthorrefmark{1},
Madeline Navarro\IEEEauthorrefmark{2},
Victor M. Tenorio\IEEEauthorrefmark{1},
Santiago Segarra\IEEEauthorrefmark{2},
Antonio G. Marques\IEEEauthorrefmark{1},
} %
\IEEEauthorblockA{
\IEEEauthorrefmark{1}Dept. of Signal Theory and Communications, King Juan Carlos University, Madrid, Spain } %
\IEEEauthorblockA{
\IEEEauthorrefmark{2}Dept. of Electrical and Computer Engineering, Rice University, Houston, TX, USA } %
}

\maketitle
\begin{abstract}
Graph neural networks (GNNs) have become a workhorse approach for learning from data defined over irregular domains, typically by implicitly assuming that the data structure is represented by a homophilic graph. However, recent works have revealed that many relevant applications involve heterophilic data where the performance of GNNs can be notably compromised. To address this challenge, we present a simple yet effective architecture designed to mitigate the limitations of the homophily assumption. The proposed architecture reinterprets the role of graph filters in convolutional GNNs, resulting in a more general architecture while incorporating a stronger inductive bias than GNNs based on filter banks. The proposed convolutional layer enhances the expressive capacity of the architecture enabling it to learn from both homophilic and heterophilic data and preventing the issue of oversmoothing. From a theoretical standpoint, we show that the proposed architecture is permutation equivariant.
Finally, we show that the proposed GNNs compares favorably relative to several state-of-the-art baselines in both homophilic and heterophilic datasets, showcasing its promising potential.
\end{abstract}

\section{Introduction}\label{S:Introduction}

While the success of graph neural networks (GNNs) is undeniable~\cite{wu2021comprehensive,zhou2020graph}, their limitations remain active research topics, such as oversmoothing, lack of interpretability, and subpar performance on heterophilic data~\cite{bo2021lowfrequencyinformationgraph,oono2021graphneuralnetworks,zhang2024trustworthy,huang2023graphlime,tenorio2024recovering}.
Since many GNNs effectively perform low-pass filtering, they may lack the expressive power needed to learn from heterophilic data, despite its ubiquity in real-world scenarios~\cite{balcilar2020bridginggapspectral,yan2022twosidessame}.
Efforts to overcome this homophily assumption often involve proposing new and more complicated architectures~\cite{zheng2024graphneuralnetworks,bo2021lowfrequencyinformationgraph,yan2022twosidessame}.

Instead, we consider a simpler approach by aggregating node features beyond a single-hop neighborhood~\cite{atwood2016diffusionconvolutionalneuralnetworks,gama2020graphs,ma2021unifiedviewgraph,xu2018representation,tenorio2021robust}.
For standard architectures that combine features from adjacent nodes, stacking more layers yields interactions for further hops; however, the inherent low-pass filtering is prone to oversmoothing~\cite{oono2021graphneuralnetworks,nt2019revisitinggraphneural,bo2021lowfrequencyinformationgraph,yan2022twosidessame}.
A natural alternative is to apply polynomial filters, which aggregate information from different hops at every layer without being constrained to low-pass representations~\cite{bianchi2021graphneuralnetworks,gama2020graphs,atwood2016diffusionconvolutionalneuralnetworks}.
Polynomial-based approaches remain popular for spectral GNNs~\cite{yan2022twosidessame,wang2022howpowerful,zhang2024lowpassfilteringlargescale}, but performing operations in the frequency domain can require heavier computation and limited flexibility~\cite{wu2021comprehensive}.
Alternatively, many works consider polynomial filters for spatial GNNs~\cite{gama2019convolutionalgraphneural,ruiz2021graph,wu2019simplifying,bianchi2021graphneuralnetworks}, which require higher-order polynomials to aggregate neighborhoods at further hops.

Polynomial graph filters not only serve as potent and flexible operators for GNNs, but they are also well-established tools in graph signal processing (GSP)~\cite{isufi2024graph,dong2020graphsignalprocessing,gama2020graphs,ruiz2021graph,rey2023robust}.
Indeed, graph filter analysis has provided rich insights and techniques for graph convolutional networks (GCNs)~\cite{ruiz2021graph,isufi2024graph,dong2020graphsignalprocessing,rey2022untrained}, which are designed with GNN layers that aggregate multi-hop information~\cite{wu2019simplifying,gasteiger2019combining,eliasof2023improvinggraphneural,rey2024convolutional}.
For example, previous works apply GSP concepts to decouple learning for transforming features at each layer and propagating features on the graph~\cite{frasca2020sign,wu2019simplifying,chen2020scalable}.

In this work, we introduce Adaptive Aggregation GCN (AAGCN), an architecture based on polynomial graph convolutional filters where we decouple training for feature transformation and feature propagation.
More specifically, we alternate between learning (i) the weights responsible for feature transformation and (ii) a single filter coefficient that controls feature aggregation at each hop.
This approach results in an intuitive graph filter-based GNN, where the frequency analyses of the filters enhances interpretability.
Moreover, AAGCN serves as a tradeoff between the original single-hop GCN~\cite{kipf2016semi} and the more expressive variant of GCNs with learned filter banks for every hop~\cite{ruiz2021graph}.
Our contributions are summarized as follows.


\begin{itemize}[left= 5pt .. 15pt, noitemsep]
    \item[1)] We propose AAGCN, a filter-based architecture for learning from homophilic and heterophilic data that alleviates oversmoothing.
    \item[2)] We introduce a decoupled training algorithm that separately learns weights for feature aggregation and feature transformation.
    \item[3)] We not only demonstrate that our method is competitive with existing baselines for benchmark graph datasets, both homophilic and heterophilic, 
    but we also showcase the explainability of AAGCN through a frequency analysis of the learned filters.
    
\end{itemize}


\section{Preliminaries and problem statement}

\subsection{Graph signal processing}\label{Ss:gsp}

We consider graph-based learning on an undirected graph $\ccalG = (\ccalV, \ccalE)$, composed of the set of $N$ nodes $\ccalV$ and the set of edges $\ccalE \subseteq \ccalV \times \ccalV$, where an edge $(i,j) \in \ccalE$ exists if and only if the nodes $i$ and $j$ are connected.
Graphical operations such as convolutions can be conveniently computed using the adjacency matrix $\bbA \in \reals^{N \times N}$, which encodes the connectivity of $\ccalG$, that is, $A_{ij} = A_{ji} \neq 0$ if and only if $(i,j) \in \ccalE$.
Note that $\bbA$ is symmetric and thus diagonalizable, which we express as $\bbA = \bbV \bbLambda \bbV^\top$, where the orthogonal matrix $\bbV$ collects the eigenvectors and $\bbLambda = \diag(\bblambda)$ the eigenvalues $\bblambda\in\reals^N$ of $\bbA$.
The spectrum of $\bbA$ plays a central role in graph spectral analysis~\cite{sandryhaila2014discrete}. 
Additionally, we consider node features, represented as a vector $\bbx\in\reals^N$, where $x_i$ is the feature value on the $i$-th node.

To process data on the graph for predictive tasks, we consider graph filters, a prominent tool in GSP that performs linear operations on graph signals~\cite{segarra2017optimal}.
In particular, a graph filter can be expressed as a polynomial of the adjacency matrix of the form
\begin{equation}\label{eq:graph_filter}
    \bbH:=\sum_{r=0}^{R-1}h_r\bbA^r.
\end{equation}
Observe that $\bbA\bbx$ shifts the graph signal $\bbx$ to neighboring nodes, and more generally $\bbA^r\bbx$ denotes a diffusion of $\bbx$ within an $r$-hop radius.
Then, $h_r$ denotes the weight of a shift over an $r$-hop neighborhood.
Thus, $\bbH$ can model the diffusion of signals over the graph and generalizes the convolution operation.
Indeed, graph filters are the basis for spectral GNNs and generalize the aggregation operation of most convolutional GNNs~\cite{isufi2024graph,dong2020graphsignalprocessing}.

\subsection{GCNs for node classification}\label{Ss:nodeclass}

We consider the prominent task of node classification.
Given a dataset $\ccalT = \{ \bbX, \ccalY_{\mathrm{tr}} \}$ containing a matrix of $F$ node features $\bbX \in \reals^{N\times F}$ on the graph $\ccalG$ and observed node labels $\ccalY_{\mathrm{tr}} \subset \ccalY$, our goal is to learn a map relating $\bbX$ and $\ccalY_{\mathrm{tr}}$ such that we may predict unseen labels $\ccalY_{\mathrm{test}} = \ccalY\backslash \ccalY_\mathrm{tr}$ from node features $\bbX$.
In particular, we obtain the non-linear parametric function $f ( \cdot, \Theta | \ccalG ) : \reals^{N \times F} \rightarrow \ccalY$, whose parameters $\Theta$ we estimate by solving
\begin{equation}\label{eq:node_class}
    \min_{\Theta} \ccalL(\ccalY_{\mathrm{tr}}, f (\bbX, \Theta | \ccalG)).
\end{equation}
Here, $\ccalL$ is a loss function chosen to suit the task at hand, such as cross entropy loss for classification or mean square error for regression when $\ccalY = \reals$.
While we focus on node classification, note that our architecture can be readily employed for other graph learning tasks.

Myriad GNNs have been proposed for solving the problem in~\eqref{eq:node_class}, including the seminal GCN in~\cite{kipf2016semi}.
The output of the model is given by the recursion of $L$ layers, with the $\ell$-th layer taking the form
\begin{equation}\label{eq:kipf_gnn}
    \bbX^{(\ell+1)} = \sigma\left(\tbA \bbX^{(\ell)} \bbW^{(\ell)}\right),
\end{equation}
where $\bbW^{(\ell)}$ is a learnable weight matrix and $\tbA = \tbD^{-\frac{1}{2}} (\bbA + \bbI) \tbD^{-\frac{1}{2}}$ is the adjacency matrix with self-loops normalized by the degree matrix $\tbD = \diag((\bbA+\bbI)\bbone)$.
While simple, the GCN in~\eqref{eq:kipf_gnn} has shown remarkable success for homophilic datasets, where neighboring nodes tend to share labels.
The convolution in~\eqref{eq:kipf_gnn} is a normalized graph filter $\Hgcn = \bbA + \bbI$, which poses critical limitations.
First, the radius of the architecture cannot exceed its depth as each layer only aggregates features from single-hop neighborhoods.
Second, as $\bbH_{\mathrm{GCN}}$ is a \emph{low-pass filter}~\cite{balcilar2020bridginggapspectral,yan2022twosidessame,bo2021lowfrequencyinformationgraph}, stacking several layers to increase the radius leads to the infamous oversmoothing problem~\cite{oono2021graphneuralnetworks,nt2019revisitinggraphneural,bo2021lowfrequencyinformationgraph,yan2022twosidessame}, where node representations become too similar to be distinguished for predictions.
Indeed, the low-pass filtering of GCNs provides intuition on the challenges of learning from heterophilic data~\cite{bo2021lowfrequencyinformationgraph,yan2022twosidessame}.


A solution to this problem was introduced in~\cite{gama2019convolutionalgraphneural,ruiz2021graph} by replacing the single-hop graph filter in \eqref{eq:kipf_gnn} with a bank of polynomial graph filters.
Instead of learning only one weight matrix $\bbW^{(\ell)}$ per layer, the filter-bank GCN learns a matrix at every hop, greatly increasing expressivity.
Formally, the $\ell$-th layer of the filter-bank GCN is
\begin{equation}\label{eq:fb_gnn}
    \bbX^{(\ell+1)} = \sigma \left( \sum_{r=0}^{R-1} \bbA^r \bbX^{(\ell)} \bbW_r^{(\ell)} \right),
\end{equation}
where the learnable weight matrix $\bbW_r^{(\ell)}$ contains the coefficients of the filter bank associated with the $r$-hop neighborhood.
Combining features at multiple hops effectively decouples the radius and depth of the architecture, as each layer can aggregate features beyond adjacent neighbors without the need of additional layers.
Furthermore, the filter bank in~\eqref{eq:fb_gnn} permits learning transformations that are not low-pass and enjoy well-established theoretical guarantees~\cite{ruiz2021graph}.

While training weights for each hop allows GCNs to adapt to heterophilic or homophilic datasets, the increased expressivity can render the network prone to overfitting.
Learning an increased number of parameters requires large amounts of data, which can cause filter-bank GCNs to perform poorly on common homophilic datasets.
However, imposing assumptions on each $\bbW_r^{(\ell)}$ can reduce the model complexity while maintaining sufficient expressivity~\cite{gama2018mimo}.
In particular, a fundamental assumption of GCNs is that edges denote similarity between nodes, typically implying similar labels or features~\cite{zheng2024graphneuralnetworks}.
Thus, it is reasonable to expect feature importance to be similar for each node, that is, similar feature weights $\bbW_r^{(\ell)}$ and $\bbW_{r+1}^{(\ell)}$ for the $r$-th and $(r+1)$-th hop neighborhoods, respectively, although the importance of each hop may vary.
We next present a GCN framework that imposes this assumption structurally, posed as a tradeoff between the single-hop GCN in~\eqref{eq:kipf_gnn} and the filter-bank GCN in~\eqref{eq:fb_gnn}.

\section{Adaptive Aggregation GCN}
We proposed the Adaptive Aggregation GCN (AAGCN), a convolutional GNN based on learnable graph filters given by the recursion
\begin{equation}\label{eq:aagnn}
    \bbX^{(\ell+1)} = \sigma \left( \sum_{r=0}^{R-1} h_r^{(\ell)} \bbA^r \bbX^{(\ell)} \bbW^{(\ell)} \right),
\end{equation}
where $h_r^{(\ell)}$ denotes coefficients of the graph filter, and $\bbX^{(\ell)} \in \reals^{ N \times F_{\ell} }$ and $\bbX^{(\ell+1)}\in\reals^{N \times F_{\ell+1}}$ contain the node representations returned by layers $\ell-1$ and $\ell$, respectively.
Intuitively, AAGCNs decouple the parameters of the architecture into two groups: (i) the coefficients $\bbh^{(\ell)}\in\reals^R$ control how node features from different hops are diffused over the graph, while (ii) the weight matrix $\bbW^{(\ell)}\in\reals^{F_{\ell+1} \times F_{\ell} }$ transforms node features and is shared by each hop.

The proposed formulation resembles existing polynomial-based GCN architectures.
Earlier works consider matrix powers for richer convolutions, but these are either less informative or more restrictive than~\eqref{eq:aagnn}~\cite{wu2019simplifying,atwood2016diffusionconvolutionalneuralnetworks,xu2019graphconvolutionalnetworks,frasca2020sign}.
Some exploit graph theory to incorporate information at further hops~\cite{gasteiger2019combining,bianchi2021graphneuralnetworks,abuelhaija2020ngcn}, but these do not learn weights for each hop.
In contrast, ours is composed of simple operations, yielding a natural compromise between \eqref{eq:kipf_gnn} and \eqref{eq:fb_gnn}.
Similar to stacking GCN layers~\eqref{eq:kipf_gnn}, a single AAGCN layer combines all node features in an $r$-hop radius to learn feature importance weights $\bbW^{(\ell)}$.
However, as in~\eqref{eq:fb_gnn}, we can aggregate beyond a single hop, preventing only low-pass filtering operations.
Moreover, AAGCN layers require $R + F_{\ell+1}F_{\ell}$ parameters, reducing complexity compared to not only the filter-bank GCN with $RF_{\ell+1}F_{\ell}$ parameters per layer but also its existing simplifications~\cite{gama2018mimo}.
While far simpler, AAGCN parameters are more influential, as each hop weight $h_r^{(\ell)}$ imposes a stronger inductive bias, mitigating the likelihood of overfitting.
Indeed, our architecture may be interpreted as requiring that each filter apply the same treatment to node features at every hop.


Beyond expressivity, AAGCNs can adapt to heterophilic or homophilic datasets by learning different weights per hop.
Crucially, we achieve this flexibility without any specific regularization term or an architecture tailored to heterophilic data.
Instead, we rely on a simple yet potent approach that exploits GSP tools~\cite{balcilar2021analyzingexpressivepower,ma2022homophily}.
As the decoupled formulation in~\eqref{eq:aagnn} models feature transformations and the graph filter in~\eqref{eq:graph_filter} separately, we require more scrupulous minimization of~\eqref{eq:node_class} to learn both sets of parameters $\Theta = (\ccalH,\ccalW)$.

\subsection{Training AAGCNs} \label{S:training}
The concept of decoupling for GNNs is not novel, but training the entirety of a decoupled model is uncommon~\cite{zhang2024lowpassfilteringlargescale}.
For example, authors in~\cite{gasteiger2019combining} approximate PageRank to diffuse learned node embeddings over the graph, and~\cite{wu2019simplifying,frasca2020sign,abuelhaija2020ngcn} precompute higher-order polynomial operations.
However, while these works indeed separate node feature diffusion from feature transformation as does~\eqref{eq:aagnn}, only the feature transformation weights are learned during training.
Differently, the AAGCN learns two sets of parameters, the filter coefficients $\ccalH = \{\bbh^{(\ell)}\}_{\ell=1}^L$ and the feature weights $\ccalW = \{\bbW^{(\ell)}\}_{\ell=1}^L$.
To avoid computing the gradient of a bilinear term, which may lead to instabilities, we train our architecture by minimizing \eqref{eq:node_class} following an alternating minimization scheme.
More precisely, denote the AAGCN as the parametric function $f (\bbX, \ccalH,\ccalW | \bbA)$.
Then, we perform the steps enumerated in Algorithm~\ref{alg:aagcn}.

\begin{algorithm}[t]
\caption{AAGCN}
\label{alg:aagcn}
\begin{algorithmic}[1]
\Require Step size $\lambda > 0$, iterations $I_{\ccalH}$, $I_{\ccalW}$
\State Initialize $\ccalH$ and $\ccalW$
\While{not terminated}
    \For{$i=1,\dots, I_{\ccalH}$}
        \State $\bbh^{(\ell)} \leftarrow \bbh^{(\ell)} - \lambda \nabla_{\bbh^{(\ell)}} \ccalL(\bbY, f(\bbX,\ccalH,\ccalW | \bbA))$
        \Statex \qquad\qquad $\forall~\ell=1,\dots,L$
    \EndFor
    \For{$i=1,\dots, I_{\ccalW}$}
        \State $\bbW^{(\ell)} \leftarrow \bbW^{(\ell)} - \lambda \nabla_{\bbW^{(\ell)}} \ccalL(\bbY, f(\bbX,\ccalH,\ccalW | \bbA))$
        \Statex \qquad\qquad $\forall~\ell=1,\dots,L$
    \EndFor
\EndWhile
\State \Return $f(\cdot, \ccalH, \ccalW | \bbA)$
\end{algorithmic}
\end{algorithm}

The stochastic gradient descent sub-algorithms from lines 3 to 8 of Algorithm~\ref{alg:aagcn} are repeated for a maximum number of $K$ epochs or until some stopping criterion is met.
If exact minimization of $\ccalL$ with respect to $\ccalH$ or $\ccalW$ is computationally prohibitive, we may choose small numbers of gradient descent iterations $I_{\ccalH}$ and $I_{\ccalW}$ for approximate solutions at each epoch.
In the numerical evaluation, we demonstrate that moderately small $I_{\ccalH}$ and $I_{\ccalW}$ achieve competitive performance despite inexact alternating minimization.

By optimizing $\ccalL$ one parameter set at a time, we split the bilinear terms in~\eqref{eq:aagnn}, resulting in computationally friendly gradients.
The number of operations per iteration of AAGCN is similar to that of the single-hop GCN~\eqref{eq:kipf_gnn}.
To see this, note that the matrix product $\bbA^R\bbX^{(\ell)} = \bbA \cdots \bbA\bbA\bbX^{(\ell)}$ can be seen as recursive multiplication by $\bbA$.
While standard matrix multiplication would yield a computational complexity of $\ccalO(RN^2F_{\ell} + NF_{\ell}F_{\ell+1})$, exploiting the sparsity of $\bbA$ can reduce computation to $\ccalO(RSF_{\ell} + NF_{\ell}F_{\ell+1})$, where $S$ represents the number of non-zero entries in $\bbA$.

\subsection{AAGCN Analysis}
The crux of AAGCNs is the graph filter, a fundamental tool in GSP that permits straightforward analyses.
First, we note that the AAGCN is clearly permutation equivariant, as we formally state next.
\begin{proposition}
    The AAGCN $f(\bbX, \ccalH, \ccalW | \bbA)$ in~\eqref{eq:aagnn} is permutation equivariant, that is, for any permutation matrix $\bbP$ it holds that
    \begin{equation}
        f(\bbP\bbX, \ccalH, \ccalW | \bbP\bbA\bbP^\top) = \bbP f (\bbX, \ccalH, \ccalW | \bbA).
    \end{equation}
\end{proposition}
\begin{proof}
    Given any permutation matrix $\bbP$ and the recursion of the AAGCN at layer $\ell$, we have that 
    \begin{equation}
        \bbP \sum_{r=0}^{R-1} h_r^{(\ell)} \bbA^r \bbP^\top  \bbP \bbX^{(\ell)} \bbW^{(\ell)}= \bbP \sum_{r=0}^{R-1} h_r^{(\ell)} \bbA^r \bbX^{(\ell)} \bbW^{(\ell)}.
    \end{equation}
    Since the non-linearity $\sigma$ is computed entry-wise, it follows that 
    \begin{equation}
        \sigma \left( \bbP \sum_{r=0}^{R-1} h_r^{(\ell)} \bbA^r \bbX^{(\ell)} \bbW^{(\ell)} \right) = \bbP \bbX^{(\ell+1)},
    \end{equation}
    and the result then follows for every layer of $f(\bbX,\ccalH,\ccalW | \bbA)$.
\end{proof}

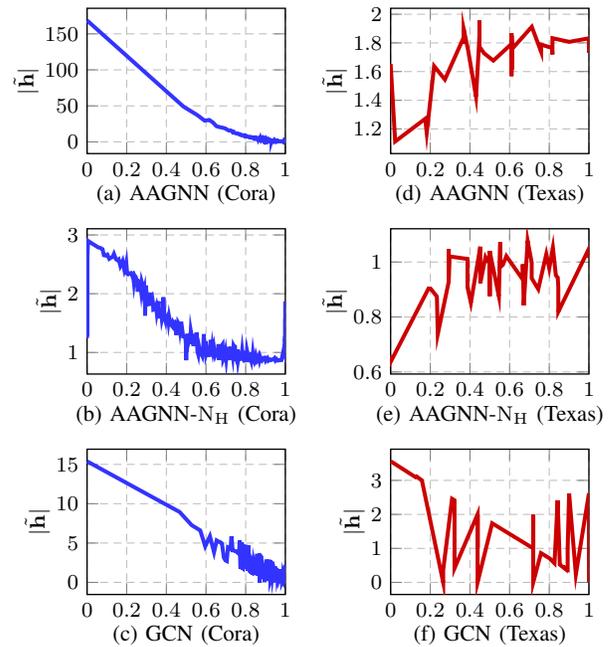
\begin{figure}[]
	\centering
		\centering
		\begin{tikzpicture}[baseline,scale=1]

\pgfplotstableread{data/20240910-1345-CoraGraphDataset-filter_freq.csv}\coraerrtable
\pgfplotstableread{data/20240910-1345-TexasDataset-filter_freq.csv}\texaserrtable

\begin{groupplot}[
    group style={group size=2 by 3,
        horizontal sep=1.4cm,
        vertical sep=1cm,},
    xlabel style={yshift=1mm},
    ylabel={$|\tbh|$},
    grid style=densely dashed,
    grid=both,
    width=120,
    height=100,
    ]    

    \nextgroupplot[xlabel={(a) AAGNN (Cora)}, xmin=0, xmax=1]
    \addplot[white!20!blue, solid] table [x=AFGNN-xax, y=AFGNN-filterfreqs] {\coraerrtable};

    \nextgroupplot[xlabel={(d) AAGNN (Texas)}, xmin=0, xmax=1]
     \addplot[black!20!red] table [x=AFGNN-xax, y=AFGNN-filterfreqs] {\texaserrtable};
     
    \nextgroupplot[xlabel={(b) AAGNN-$\mathrm{N_H}$ (Cora)}, xmin=0, xmax=1]
    \addplot[white!20!blue, solid] table [x=AFGNN-normH-xax, y=AFGNN-normH-filterfreqs] {\coraerrtable};

    \nextgroupplot[xlabel={(e) AAGNN-$\mathrm{N_H}$ (Texas)}, xmin=0, xmax=1]
     \addplot[black!20!red] table [x=AFGNN-normH-xax, y=AFGNN-normH-filterfreqs] {\texaserrtable};

     \nextgroupplot[xlabel={(c) GCN (Cora)}, xmin=0, xmax=1]
    \addplot[white!20!blue, solid] table [x=Kipf-xax, y=Kipf-filterfreqs] {\coraerrtable};

    \nextgroupplot[xlabel={(f) GCN (Texas)}, xmin=0, xmax=1]
     \addplot[black!20!red] table [x=Kipf-xax, y=Kipf-filterfreqs] {\texaserrtable};

\end{groupplot}
\end{tikzpicture}
    \vspace{-1mm}
	\caption{Graph filter frequency response for AAGCNs versus single-hop GCNs in~\eqref{eq:kipf_gnn} when learning from a homophilic (Cora) or heterophilic (Texas) dataset. The frequencies are ordered as in~\cite{sandryhaila2014discrete}.
    }
    \label{fig:freq_resps}
    \vspace{-3mm}
\end{figure}

\begin{figure*}[!t]
    \centering
    \hspace{-.2cm}
    \begin{minipage}[c]{.6\textwidth}
        \scriptsize
        \centering
        \begin{tabular}{l@{\hspace{2pt}}|c|c|c|c|c}
         & Texas (0.11) & Wisconsin (0.2) & Cornell (0.13) & Cora (0.81) & Citeseer (0.74) \\
         \hline
        AAGCN         & 0.838 ± 0.055 & \textbf{0.855 ± 0.035} & \textbf{0.776 ± 0.068} & 0.727 ± 0.026   & 0.610 ± 0.036  \\
        AAGCN-$\mathrm{N_A}$   & \textbf{0.855 ± 0.051} & \textbf{0.853 ± 0.042} & 0.762 ± 0.034 & \textbf{0.805 ± 0.004}   & 0.697 ± 0.007  \\
        AAGCN-$\mathrm{N_H}$   & \textbf{0.857 ± 0.051} & 0.847 ± 0.034 & \textbf{0.773 ± 0.044} & 0.792 ± 0.007   & {0.699 ± 0.006}  \\
        \hline
        FB-GCN         & 0.757 ± 0.048 & 0.761 ± 0.039 & 0.743 ± 0.044 & 0.668 ± 0.063   & 0.566 ± 0.023  \\
        FB-GCN-$\mathrm{N_A}$    & 0.797 ± 0.022 & 0.845 ± 0.030 & 0.751 ± 0.043 & 0.804 ± 0.003   & {0.698 ± 0.010}  \\
        \hline
        GCN          & 0.576 ± 0.048 & 0.445 ± 0.083 & 0.389 ± 0.095 & 0.753 ± 0.015   & 0.640 ± 0.009  \\
        GCN-$\mathrm{N_A}$  & 0.592 ± 0.055 & 0.533 ± 0.063 & 0.411 ± 0.051 & \textbf{0.808 ± 0.006}   & \textbf{0.707 ± 0.006}  \\
        MLP           & 0.814 ± 0.057 & 0.839 ± 0.036 & 0.754 ± 0.056 & 0.421 ± 0.056   & 0.453 ± 0.035  \\
        GAT           & 0.643 ± 0.063 & 0.641 ± 0.089 & 0.549 ± 0.047 & 0.744 ± 0.030   & 0.641 ± 0.014  \\
        SAGE     & 0.784 ± 0.042 & 0.765 ± 0.044 & 0.646 ± 0.063 & 0.800 ± 0.010   & \textbf{0.700 ± 0.011}  \\
        GIN           & 0.649 ± 0.055 & 0.516 ± 0.082 & 0.446 ± 0.085 & 0.689 ± 0.028   & 0.527 ± 0.043  \\
        \hline
        \end{tabular}
        \captionof{table}{Mean accuracy and standard deviation when solving a node classification problem for heterophilic and homophilic datasets. The two best performing architectures for each dataset are highlighted in \textbf{boldface}.}
        \label{tab:accuracy_metrics}
    \end{minipage}
    \hfill
    \begin{minipage}[c]{.39\textwidth}
        \centering
        \begin{tikzpicture}[baseline,scale=.55]

\pgfplotsset{colormap={CM}{
    rgb255=(54, 75, 154)
    rgb255=(74, 123, 183)
    rgb255=(110, 166, 205)
    rgb255=(152, 202, 225)
    rgb255=(253, 179, 102)
    rgb255=(246, 126, 75)
    rgb255=(221, 61, 45)
    rgb255=(165, 0, 38)
}}

\begin{groupplot}[
    name=ablation_study,
    table/col sep=space,
    width=5.5cm,
    height=5.5cm,
    group style={group size=2 by 1,
        horizontal sep=1.6cm,
        vertical sep=1.8cm,},
    enlargelimits=false,
    colorbar horizontal,
    xtick={0,1,2,3,4},
    xticklabels={1,5,10,25,50},
    ytick={0,1,2,3,4},
    yticklabels={1,5,10,25,50},
    axis on top,
    label style={font=\LARGE},
    tick style={draw=none},
    tick label style={font=\Large},
    title style={font=\huge}
    ]

    \pgfplotstableread{data/ablation/Ablation_CornellDataset-A-GCN.csv}\matrixA
    \pgfplotstableread{data/ablation/Ablation_CornellDataset-H-GCN-normH.csv}\matrixB

\nextgroupplot[xlabel={No. epochs $I_{\ccalW}$},ylabel={No. epochs $I_{\ccalH}$},title={(a) AAGNN},xmin=-0.5,xmax=4.5,ymin=-0.5,ymax=4.5]
    \addplot [
        matrix plot*,
        colorbar,
        nodes near coords,
        every node near coord/.append style={
            font=\footnotesize,
            /pgf/number format/fixed,        
            /pgf/number format/precision=2,  
            text=white,
            anchor=center
        },
        point meta=explicit,
        mesh/cols=5,  
    ] table [meta=value] {\matrixA};

\nextgroupplot[xlabel={No. epochs $I_{\ccalW}$},ylabel={No. epochs $I_{\ccalH}$},title={(b) AAGNN-$\mathrm{N_H}$},xmin=-0.5,xmax=4.5,ymin=-0.5,ymax=4.5]
    \addplot [
        matrix plot*,
        colorbar,
        nodes near coords,
        every node near coord/.append style={
            font=\footnotesize,
            /pgf/number format/fixed,        
            /pgf/number format/precision=2,  
            text=white,
            anchor=center,
        },
        point meta=explicit,
        mesh/cols=5,  
    ] table [meta=value] {\matrixB};

\end{groupplot}
\end{tikzpicture}
        \caption{Ablation study comparing performance.}
        \label{fig:ablation}
        \vspace{-.1cm}
    \end{minipage}
    \vspace{-.4cm}
\end{figure*}

Beyond simplification, our decoupling approach yields further advantages.
By learning the filter coefficients $\bbh^{(\ell)}$, we can analyze the type of graph filter learned by the architecture, a common task in GSP~\cite{sandryhaila2014discrete,rey2023robust}.
The frequency response of a graph filter is given by $\tbh = \bbPsi \bbh$, where $\bbPsi \in \reals^{N \times R}$ is the Vandermonde matrix with entries $\Psi_{ij} = \lambda_i^j$, with $\lambda_i$ being the $i$-th eigenvalue of $\bbA$~\cite{segarra2017optimal}.
If $\tbh$ is low for higher frequencies, we obtain a low-pass graph filter, assumed by most GCN architectures~\cite{kipf2016semi,oono2021graphneuralnetworks,nt2019revisitinggraphneural}, but heterophilic datasets often require high-frequency information for effective predictions~\cite{bo2021lowfrequencyinformationgraph,yan2022twosidessame}.

We compare the frequency responses of filters learned via AAGCN versus the single-hop GCN with fixed filter $\Hgcn$ (c.f. ~\eqref{eq:kipf_gnn}~\cite{kipf2016semi}) using the datasets Cora and Texas~\cite{mccallum2000automating,craven1998learning}.
As expected, Figs.~\ref{fig:freq_resps}a to c show that for both AAGCN and the traditional GCN, most of the energy concentrates at lower frequencies, depicting a \textit{low-pass} graph filter for the homophilic Cora.
In contrast, the heterophilic Texas dataset causes the energy to concentrate in high-frequency eigenvalues for AAGCN in Figs~\ref{fig:freq_resps}d and e, that is, a \textit{high-pass} filter.
However, in \cref{fig:freq_resps}f the GCN in~\eqref{eq:kipf_gnn} returns a low-pass filter, even for the heterophilic dataset.
Thus, we showed that AAGCN is adaptable to homophilic or heterophilic data and we can interpret the learned frequency response thanks to decoupling $\ccalH$ from $\ccalW$.
This behavior is most evident in Figs.~\ref{fig:freq_resps}b and e, which apply a variant of AAGCN described next.

\subsection{Normalizing the AAGCN}\label{sec:aagnn_var}

The proposed formulation in~\eqref{eq:aagnn} is sufficient for graph data where node degrees do not vary greatly.
However, in practice graphs often exhibit exponential degree distributions, which may cause preferential learning for only high-degree nodes.
We thus propose normalized variants of AAGCN to promote equitable learning from all nodes. 


First, we may replace the adjacency matrix $\bbA$ with its degree-normalized counterpart $\bar{\bbA} = \bbD^{-\frac{1}{2}} \bbA \bbD^{-\frac{1}{2}}$, where $\bbD$ is a diagonal matrix collecting the node degrees.
Normalizing the adjacency matrix is a common and straightforward approach, used in~\eqref{eq:kipf_gnn}. 
We refer to this as AAGCN-$\mathrm{N_A}$ for AAGCN with normalized adjacency matrix $\bar{\bbA}$.
The variant AAGCN-$\mathrm{N_A}$ indeed balances the connections for each node by degree; however, we may apply a stronger version to ensure that nodes receive balanced treatment from the resultant graph filter.


The recursion in~\eqref{eq:aagnn} can be interpreted as a variant of~\eqref{eq:kipf_gnn} where we replace $\bbA+\bbI$ with the filter $\bbH^{(\ell)} = \sum_{r=0}^{R-1} h_r^{(\ell)} \bbA^r$.
From this viewpoint, instead of normalizing the adjacency matrix $\bbA$, we can normalize the learned graph filter as $\bar{\bbH}^{(\ell)} = (\bbD_\bbH^{(\ell)})^{-\frac{1}{2}}\bbH^{(\ell)}(\bbD_\bbH^{(\ell)})^{-\frac{1}{2}}$ with $\bbD_\bbH^{(\ell)} := \diag(\bbH^{(\ell)}\bbone)$.
We let this variant be denoted AAGCN-$\mathrm{N_H}$ as we normalize the graph filter $\bbH^{(\ell)}$ for each layer $\ell = 1,\dots, L$.
These alternative architectures AAGCN-$\mathrm{N_A}$ and AAGCN-$\mathrm{N_H}$ can help in different settings, which we demonstrate in the next section.

\section{Numerical experiments}\label{S:exps}

We now assess the performance of the proposed AAGCN architecture on both homophilic and heterophilic datasets to highlight both its versatility and efficiency in diverse scenarios.
The code with our proposed AAGCN and our experiments is available on GitHub\footnote{\url{https://github.com/reysam93/adaptive_agg_gcn}}.

\vspace{0.2mm}
\noindent
\textbf{Datasets}.
We compare several architectures for node classification over 5 different datasets.
Texas, Cornell, and Wisconsin are \textit{heterophilic}.
They contain subgraphs from the WebKB graph dataset~\cite{craven1998learning}, where nodes represent webpages for members of university computer science departments, edges represent hyperlinks between them, and node labels represent the role of each member in the university. 
Cora and Citeseer are \textit{homophilic} graph datasets, specifically citation networks~\cite{mccallum2000automating}, where nodes represent scientific articles, edges represent citations among them, and node labels encode the topic of the article.
The homophily score~\cite{lim2021large} of each of the dataset is shown in the first row of Table~\ref{tab:accuracy_metrics}, where a value closer to 1 implies higher homophily.

\vspace{2mm}
\noindent
\textbf{Baselines}.
We compare the performance of the proposed architectures against several widely used baselines in the literature, including GCN~\cite{kipf2016semi} [see \eqref{eq:kipf_gnn}], FB-GCN~\cite{ruiz2021graph} [see \eqref{eq:fb_gnn}], GAT~\cite{velivckovic2018graph}, graph SAGE~\cite{hamilton2017inductive}, GIN~\cite{xu2018powerful}, and a graph-agnostic MLP.
Our proposed architecture is labeled as ``AAGCN'' [see~\eqref{eq:aagnn}], where ``AAGCN-$\mathrm{N_A}$'' represents the architecture in equation~\eqref{eq:aagnn} with the degree-normalized adjacency matrix $\bar{\bbA}^{(\ell)}$ and ``AAGCN-$\mathrm{N_H}$''  the architecture with normalized graph filters $\bar{\bbH}^{(\ell)}$.

\vspace{2mm}
\noindent
\textbf{Main results}.
The performances of the proposed architectures and the baselines are summarized in Table~\ref{tab:accuracy_metrics}.
First, we see that all our AAGCN variants outperform all alternatives on the heterophilic datasets.
This corroborates our claim that by using a polynomial graph filter the AAGCN is not restricted to learning smooth relations, which is critical when learning from heterophilic data.
Second, on homophilic datasets, our architectures achieve performance comparable to that of the baselines, underscoring the versatility of AAGCN implementations across different types of data.
Finally, compared to the architecture in equation~\eqref{eq:fb_gnn} (both ``FB-GCN'' and its counterpart using the normalized adjacency ``FB-GCN-$\mathrm{N_A}$''), our proposed methods achieve a greater accuracy.
Thus, our AAGCN architecture is indeed flexible enough to compete with the highly expressive filter-bank GCNs by imposing a stronger inductive bias while requiring far fewer parameters.

\vspace{2mm}
\noindent
\textbf{Ablation study}.
As mentioned in~\cref{S:training} and depicted in Algorithm~\ref{alg:aagcn}, we train our architecture in an iterative manner to avoid the instabilities resulting from the bilinear term in the learnable parameters of the architecture.
This experiment empirically verifies the validity of this claim by comparing the number gradient descent epochs $I_{\ccalH}$ and $I_{\ccalW}$ used to update the filter coefficients $\ccalH$ and feature weights $\ccalW$, respectively.
The results are shown in \cref{fig:ablation}, where each value denotes the difference in accuracy between the non-iterative version of the algorithm (training all parameters at once) and the iterative version presented in~\cref{alg:aagcn}.
Clearly, training the AAGCN following an approximate alternating optimization scheme yields an advantage over the classical training paradigm. 
Moreover, we observe that increasing both numbers of epochs $I_{\ccalH}$ and $I_{\ccalW}$ in the inner loops related to $\ccalH$ and $\ccalW$, respectively, is beneficial for the two architectures analyzed.
Note that a larger $I_{\ccalH}$ and $I_{\ccalW}$ amounts to minimizing $\ccalL$ with respect to each variable with a better accuracy, but we only require a small number of epochs to achieve competitive and even superior performance to the non-iterative version.

\section{Conclusion}\label{S:conc}
We introduced AAGCN, a simple yet powerful GNN that reinterprets the role of graph filters in GCNs.
Compared to existing convolutional architectures, AAGCN strikes a balance between simpler models that perform low-pass filtering at every layer and more complex ones that rely on banks of graph filters, which require a large number of parameters and may lead to overfitting. 
Key to the success of AAGCN, we introduced a novel training paradigm based on alternating optimization.
This approach decouples the estimation of the filter coefficients from the remaining learnable parameters, resulting in a more stable computation of the gradient updates.
Furthermore, we showed that the AAGCN is a permutation-equivariant architecture.
By analysis of learned filter frequency responses, we empirically showed that AAGCN can learn both low-pass and high-pass filters, depending on whether the data is homophilic or heterophilic.
Finally, we validated the performance of the proposed architecture and the impact of the proposed training scheme across multiple datasets.
In future work, we will expand on the interpretability of this simplified architecture, continuing the frequency response analysis.
In addition, we will consider other interesting properties such as the stability and transferability of AAGCNs~\cite{ruiz2021graph}.

\newpage

\bibliographystyle{IEEEbib}
\bibliography{myIEEEabrv,biblio}

\begin{thebibliography}{10}

\bibitem{wu2021comprehensive}
Z.~Wu, S.~Pan, F.~Chen, G.~Long, C.~Zhang, and P.~S. Yu,
\newblock ``A comprehensive survey on graph neural networks,''
\newblock {\em IEEE Trans. Neural Netw. Learn. Syst.}, vol. 32, no. 1, pp.
  4--24, 2021.

\bibitem{zhou2020graph}
J.~Zhou, G.~Cui, S.~Hu, Z.~Zhang, C.~Yang, Z.~Liu, L.~Wang, C.~Li, and M.~Sun,
\newblock ``Graph neural networks: A review of methods and applications,''
\newblock {\em AI Open}, vol. 1, pp. 57--81, 2020.

\bibitem{bo2021lowfrequencyinformationgraph}
D.~Bo, X.~Wang, C.~Shi, and H.~Shen,
\newblock ``Beyond low-frequency information in graph convolutional networks,''
\newblock {\em AAAI Conf. on Artif. Intell.}, vol. 35, no. 5, pp. 3950--3957,
  2021.

\bibitem{oono2021graphneuralnetworks}
K.~Oono and T.~Suzuki,
\newblock ``Graph neural networks exponentially lose expressive power for node
  classification,'' 2021.

\bibitem{zhang2024trustworthy}
H.~Zhang, B.~Wu, X.~Yuan, S.~Pan, H.~Tong, and J.~Pei,
\newblock ``Trustworthy graph neural networks: Aspects, methods and trends,''
\newblock {\em Proc. IEEE}, vol. 112, no. 2, pp. 97--139, 2024.

\bibitem{huang2023graphlime}
Q.~Huang, M.~Yamada, Y.~Tian, D.~Singh, and Y.~Chang,
\newblock ``{{GraphLIME}}: Local interpretable model explanations for graph
  neural networks,''
\newblock {\em IEEE Trans. Knowledge and Data Engineering}, vol. 35, no. 7, pp.
  6968--6972, 2023.

\bibitem{tenorio2024recovering}
V.~M. Tenorio, M.~Navarro, S.~Segarra, and A.~G. Marques,
\newblock ``Recovering missing node features with local structure-based
  embeddings,''
\newblock in {\em IEEE Intl. Conf. Acoust., Speech and Signal Process.
  (ICASSP)}. IEEE, 2024, pp. 9931--9935.

\bibitem{balcilar2020bridginggapspectral}
M.~Balcilar, G.~Renton, P.~Heroux, B.~Gauzere, S.~Adam, and P.~Honeine,
\newblock ``Bridging the gap between spectral and spatial domains in graph
  neural networks,''
\newblock {\em arXiv preprint arXiv:2003.11702}, 2020.

\bibitem{yan2022twosidessame}
Y.~Yan, M.~Hashemi, K.~Swersky, Y.~Yang, and D.~Koutra,
\newblock ``Two sides of the same coin: Heterophily and oversmoothing in graph
  convolutional neural networks,''
\newblock in {\em IEEE Intl. Conf. Data Mining Wrkshps. (ICDMW)}. 2022, pp.
  1287--1292, IEEE.

\bibitem{zheng2024graphneuralnetworks}
X.~Zheng, Y.~Wang, Y.~Liu, M.~Li, M.~Zhang, D.~Jin, P.~S. Yu, and S.~Pan,
\newblock ``Graph neural networks for graphs with heterophily: A survey,''
\newblock {\em arXiv:2202.07082}, 2024.

\bibitem{atwood2016diffusionconvolutionalneuralnetworks}
J.~Atwood and D.~Towsley,
\newblock ``Diffusion-convolutional neural networks,''
\newblock in {\em Advances in Neural Info. Process. Syst.}, 2016, vol.~29.

\bibitem{gama2020graphs}
F.~Gama, E.~Isufi, G.~Leus, and A.~Ribeiro,
\newblock ``Graphs, convolutions and neural networks: From graph filters to
  graph neural networks,''
\newblock {\em IEEE Signal Process. Mag.}, vol. 37, no. 6, pp. 128--138, 2020.

\bibitem{ma2021unifiedviewgraph}
Y.~Ma, X.~Liu, T.~Zhao, Y.~Liu, J.~Tang, and N.~Shah,
\newblock ``A unified view on graph neural networks as graph signal
  denoising,''
\newblock in {\em ACM Intl. Conf. Info. \& Knowledge Management}. 2021, pp.
  1202--1211, ACM.

\bibitem{xu2018representation}
K.~Xu, C.~Li, Y.~Tian, T.~Sonobe, K.~Kawarabayashi, and S.~Jegelka,
\newblock ``Representation learning on graphs with jumping knowledge
  networks,''
\newblock in {\em Intl. Conf. on Machine Learning (ICML)}. 2018, vol.~80 of
  {\em Proceedings of Machine Learning Research}, pp. 5453--5462, PMLR.

\bibitem{tenorio2021robust}
V.~M. Tenorio, S.~Rey, F.~Gama, S.~Segarra, and A.~G. Marques,
\newblock ``A robust alternative for graph convolutional neural networks via
  graph neighborhood filters,''
\newblock in {\em Conf. Signals, Syst., Computers (Asilomar)}. IEEE, 2021, pp.
  1573--1578.

\bibitem{nt2019revisitinggraphneural}
H.~NT and T.~Maehara,
\newblock ``Revisiting graph neural networks: All we have is low-pass
  filters,''
\newblock {\em arXiv preprint arXiv:1905.09550}, 2019.

\bibitem{bianchi2021graphneuralnetworks}
F.~M. Bianchi, D.~Grattarola, L.~Livi, and C.~Alippi,
\newblock ``Graph neural networks with convolutional filters,''
\newblock {\em IEEE Trans. Pattern Analysis and Machine Intell.}, pp. 1--1,
  2021.

\bibitem{wang2022howpowerful}
X.~Wang and M.~Zhang,
\newblock ``How powerful are spectral graph neural networks,''
\newblock in {\em Intl. Conf. on Machine Learning (ICML)}. 2022, vol. 162 of
  {\em Proc. of Machine Learning Research}, pp. 23341--23362, PMLR.

\bibitem{zhang2024lowpassfilteringlargescale}
Q.~Zhang, J.~Li, Y.~Sun, S.~Wang, J.~Gao, and B.~Yin,
\newblock ``Beyond low-pass filtering on large-scale graphs via adaptive
  filtering graph neural networks,''
\newblock {\em Neural Networks}, vol. 169, pp. 1--10, 2024.

\bibitem{gama2019convolutionalgraphneural}
F.~Gama, A.~G. Marques, G.~Leus, and A.~Ribeiro,
\newblock ``Convolutional graph neural networks,''
\newblock in {\em Conf. Signals, Syst., Computers (Asilomar)}. 2019, pp.
  452--456, IEEE.

\bibitem{ruiz2021graph}
L.~Ruiz, F.~Gama, and A.~Ribeiro,
\newblock ``Graph neural networks: Architectures, stability, and
  transferability,''
\newblock {\em Proc. IEEE}, vol. 109, no. 5, pp. 660--682, 2021.

\bibitem{wu2019simplifying}
F.~Wu, A.~Souza, T.~Zhang, C.~Fifty, T.~Yu, and K.~Weinberger,
\newblock ``Simplifying graph convolutional networks,''
\newblock in {\em Intl. Conf. on Machine Learning (ICML)}. 2019, vol.~97 of
  {\em Proceedings of Machine Learning Research}, pp. 6861--6871, PMLR.

\bibitem{isufi2024graph}
E.~Isufi, F.~Gama, D.~I. Shuman, and S.~Segarra,
\newblock ``Graph filters for signal processing and machine learning on
  graphs,''
\newblock {\em IEEE Trans. Signal Process.}, pp. 1--32, 2024.

\bibitem{dong2020graphsignalprocessing}
X.~Dong, D.~Thanou, L.~Toni, M.~Bronstein, and P.~Frossard,
\newblock ``Graph signal processing for machine learning: A review and new
  perspectives,''
\newblock {\em IEEE Signal Process. Mag.}, vol. 37, no. 6, pp. 117--127, 2020.

\bibitem{rey2023robust}
S.~Rey, V.~M. Tenorio, and A.~G. Marques,
\newblock ``Robust graph filter identification and graph denoising from signal
  observations,''
\newblock {\em {IEEE} Trans. Signal Process.}, vol. 71, pp. 3651--3666, 2023.

\bibitem{rey2022untrained}
S.~Rey, S.~Segarra, R.~Heckel, and A.~G. Marques,
\newblock ``Untrained graph neural networks for denoising,''
\newblock {\em {IEEE} Trans. Signal Process.}, vol. 70, pp. 5708--5723, 2022.

\bibitem{gasteiger2019combining}
J.~Gasteiger, A.~Bojchevski, and S.~G{\"u}nnemann,
\newblock ``Combining neural networks with personalized {{PageRank}} for
  classification on graphs,''
\newblock in {\em Intl. Conf. on Learning Representations (ICLR)}, 2019.

\bibitem{eliasof2023improvinggraphneural}
M.~Eliasof, L.~Ruthotto, and E.~Treister,
\newblock ``Improving graph neural networks with learnable propagation
  operators,''
\newblock in {\em Intl. Conf. on Machine Learning (ICML)}. 2023, vol. 202 of
  {\em Proceedings of Machine Learning Research}, pp. 9224--9245, PMLR.

\bibitem{rey2024convolutional}
S.~Rey, H.~Ajorlou, and G.~Mateos,
\newblock ``Convolutional learning on directed acyclic graphs,''
\newblock {\em arXiv preprint arXiv:2405.03056}, 2024.

\bibitem{frasca2020sign}
F.~Frasca, E.~Rossi, D.~Eynard, B.~Chamberlain, M.~Bronstein, and F.~Monti,
\newblock ``{SIGN}: Scalable inception graph neural networks,''
\newblock {\em arXiv preprint arXiv:2004.11198}, 2020.

\bibitem{chen2020scalable}
M.~Chen, Z.~Wei, B.~Ding, Y.~Li, Y.~Yuan, X.~Du, and J.~Wen,
\newblock ``Scalable graph neural networks via bidirectional propagation,''
\newblock in {\em Advances in Neural Info. Process. Syst.}, 2020, vol.~33, pp.
  14556--14566.

\bibitem{kipf2016semi}
T.~N. Kipf and M.~Welling,
\newblock ``Semi-supervised classification with graph convolutional networks,''
\newblock in {\em Intl. Conf. on Learning Representations (ICLR)}, 2017.

\bibitem{sandryhaila2014discrete}
A.~Sandryhaila and J.~M.~F. Moura,
\newblock ``Discrete signal processing on graphs: Frequency analysis,''
\newblock {\em {IEEE} Trans. Signal Process.}, vol. 62, no. 12, pp. 3042--3054,
  2014.

\bibitem{segarra2017optimal}
S.~Segarra, A.~G. Marques, and A.~Ribeiro,
\newblock ``Optimal graph-filter design and applications to distributed linear
  network operators,''
\newblock {\em {IEEE} Trans. Signal Process.}, vol. 65, no. 15, pp. 4117--4131,
  2017.

\bibitem{gama2018mimo}
F.~Gama, A.~G. Marques, A.~Ribeiro, and G.~Leus,
\newblock ``{{MIMO}} graph filters for convolutional neural networks,''
\newblock in {\em IEEE Intl. Wrkshp. on Signal Process. Advances in Wireless
  Commun. (SPAWC)}. 2018, pp. 1--5, IEEE.

\bibitem{xu2019graphconvolutionalnetworks}
B.~Xu, H.~Shen, Q.~Cao, K.~Cen, and X.~Cheng,
\newblock ``Graph convolutional networks using heat kernel for semi-supervised
  learning,''
\newblock in {\em Intl. Joint Conf. on Artif. Intell.}, 2019, pp. 1928--1934.

\bibitem{abuelhaija2020ngcn}
S.~Abu-El-Haija, A.~Kapoor, B.~Perozzi, and J.~Lee,
\newblock ``{N-GCN}: Multi-scale graph convolution for semi-supervised node
  classification,''
\newblock in {\em Uncertainty in Artif. Intell. Conf.} 2020, vol. 115 of {\em
  Proceedings of Machine Learning Research}, pp. 841--851, PMLR.

\bibitem{balcilar2021analyzingexpressivepower}
M.~Balcilar, G.~Renton, P.~H{\'e}roux, B.~Ga{\"u}z{\`e}re, S.~Adam, and
  P.~Honeine,
\newblock ``Analyzing the expressive power of graph neural networks in a
  spectral perspective,''
\newblock in {\em Intl. Conf. on Learning Representations (ICLR)}, 2021.

\bibitem{ma2022homophily}
Y.~Ma, X.~Liu, N.~Shah, and J.~Tang,
\newblock ``Is homophily a necessity for graph neural networks?,''
\newblock in {\em Intl. Conf. on Learning Representations (ICLR)}, 2022.

\bibitem{mccallum2000automating}
A.~K. McCallum, K.~Nigam, J.~Rennie, and K.~Seymore,
\newblock ``Automating the construction of internet portals with machine
  learning,''
\newblock {\em Info. Retrieval}, vol. 3, pp. 127--163, 2000.

\bibitem{craven1998learning}
M.~Craven, D.~DiPasquo, D.~Freitag, A.~McCallum, T.~Mitchell, K.~Nigam, and
  S.~Slattery,
\newblock ``Learning to extract symbolic knowledge from the {World Wide Web},''
\newblock {\em AAAI Conf. on Artif. Intell.}, vol. 3, no. 3.6, pp. 2, 1998.

\bibitem{lim2021large}
D.~Lim, F.~Hohne, X.~Li, S.~L. Huang, V.~Gupta, O.~Bhalerao, and S.~N. Lim,
\newblock ``Large scale learning on non-homophilous graphs: New benchmarks and
  strong simple methods,''
\newblock {\em Conf. Neural Inform. Process. Syst.}, vol. 34, pp. 20887--20902,
  2021.

\bibitem{velivckovic2018graph}
P.~Veli{\v{c}}kovi{\'c}, G.~Cucurull, A.~Casanova, A.~Romero, P.~Li{\`o}, and
  Y.~Bengio,
\newblock ``Graph attention networks,''
\newblock in {\em Intl. Conf. on Learning Representations (ICLR)}, 2018.

\bibitem{hamilton2017inductive}
W.~Hamilton, Z.~Ying, and J.~Leskovec,
\newblock ``Inductive representation learning on large graphs,''
\newblock {\em Conf. Neural Inform. Process. Syst.}, vol. 30, 2017.

\bibitem{xu2018powerful}
K.~Xu, W.~Hu, J.~Leskovec, and S.~Jegelka,
\newblock ``How powerful are graph neural networks?,''
\newblock in {\em Intl. Conf. on Learning Representations (ICLR)}, 2019.

\end{thebibliography}

\end{document}